%% file: arxiv.tex
\documentclass{article}

\usepackage{microtype}
\usepackage{graphicx}
\usepackage{subcaption}
\usepackage{booktabs} 

\usepackage{hyperref}

\usepackage[preprint]{icml2026}

\usepackage{amsmath}
\usepackage{amssymb}
\usepackage{mathtools}
\usepackage{amsthm}
\usepackage{thm-restate}

\usepackage[capitalize,noabbrev]{cleveref}

\theoremstyle{plain}
\newtheorem{theorem}{Theorem}[section]

\newtheorem{lemma}[theorem]{Lemma}
\newtheorem{corollary}[theorem]{Corollary}
\theoremstyle{definition}
\newtheorem{definition}[theorem]{Definition}
\newtheorem{assumption}[theorem]{Assumption}
\theoremstyle{remark}
\newtheorem{remark}[theorem]{Remark}

\usepackage[textsize=tiny]{todonotes}

\input{macros}

\icmltitlerunning{Active PU learning}

\begin{document}

\twocolumn[
  \icmltitle{Active learning from positive and unlabeled examples}



  \icmlsetsymbol{equal}{*}

  \begin{icmlauthorlist}
    \icmlauthor{Farnam Mansouri}{waterloo,vector}
    \icmlauthor{Sandra Zilles}{regina,amii}
    \icmlauthor{Shai Ben-David}{waterloo,vector}
  \end{icmlauthorlist}

  \icmlaffiliation{waterloo}{David R. Cheriton School of Computer Science, University of Waterloo, Waterloo, ON, Canada}
  \icmlaffiliation{vector}{Vector Institute, Toronto, ON, Canada}
  \icmlaffiliation{regina}{Department of Computer Science, University of Regina, Regina, SK, Canada}
  \icmlaffiliation{amii}{Alberta Machine Intelligence Institute (AMII), Edmonton, AB, Canada}

  \icmlcorrespondingauthor{Farnam Mansouri}{f5mansou@uwaterloo.ca}

  \icmlkeywords{Machine Learning, ICML}

  \vskip 0.3in
]



\printAffiliationsAndNotice{}  

\begin{abstract}
Learning from positive and unlabeled data (PU learning) is a weakly supervised variant of binary classification in which the learner receives labels only for (some) positively labeled instances, while all other examples remain unlabeled. Motivated by applications such as advertising and anomaly detection, we study an \emph{active} PU learning setting where the learner can adaptively query instances from an unlabeled pool, but a queried label is revealed only when the instance is positive and an independent coin flip succeeds; otherwise the learner receives no information. In this paper, we provide the first theoretical analysis of the label complexity of active PU learning.

\end{abstract}

\section{Introduction}

Learning from positive and unlabeled data (PU learning) is a weakly supervised variant of binary classification in which the training data only consists of positively labeled and unlabeled examples. PU learning appears naturally in many real-world problems, including personal advertisement,
land cover classification \cite{li2010positive}, 
prediction of protein similarity \cite{elkan2008learning}, as well as applications such as knowledge base completion \cite{bekker2020learning} and disease-gene identification \cite{yang2012positive}.


In active learning, a learning algorithm is given access to a large pool of unlabeled examples, and is allowed to request the label of any particular example from that pool. This is in contrast with standard passive learning, where the data are assumed to be collected randomly and labeled independently of the learner's choices. The objective of active learning is to learn an accurate classifier while requesting as few labels as possible, thus in particular reducing the workload of human annotators by carefully selecting the examples from the unlabeled pool that should be labeled.

Just as in standard machine learning, the labeling process can also be difficult for many PU learning applications. For example, in recommender systems, labels are often collected based on whether a recommendation is relevant (a positive outcome), while non-interactions are typically ambiguous and treated as unlabeled. Similarly, in anomaly detection, since anomalies are rare, the learner may only have access to data labeled as normal (which we view as the positive class). In both settings, obtaining labels can be costly, and may potentially benefit from more carefully selecting which examples from the unlabeled pool to query \cite{perini2020class, vercruyssen2018semi}.

We therefore study \emph{active PU learning}, in which the learner may adaptively query instances, but the label request is given to an expert that is only able to detect positive labels for a portion of instances. We make the simplifying assumption that, for each queried positive instance, the expert detects its label independently with probability~$\omega$. This assumption is equivalent to the \emph{selected-completely-at-random} (SCAR) assumption, which is often employed in passive PU learning \cite{liu2002partially, blanchard2010semi, du2015convex, bekker2018estimating, mansouri2025learning}.

Since the objective of active learning is to identify an accurate classifier with as small a number of queries as possible, a standard measure of evaluation in active learning is the so-called \emph{label complexity}, i.e., the number of label requests that are necessary and sufficient to learn an accurate classifier. In this paper, we provide the first label complexity analysis for active PU learning, using a disagreement-based analysis. Our bounds depend critically on the \emph{disagreement coefficient} $\theta$, a standard quantity in the theory of active learning \cite{dasgupta2007agnostic, hanneke2007bound, hanneke2009theoretical, hanneke2011rates, hanneke2014theory}. They also rely on a lower bound on the positive class prior $\pi_{\D}$, without requiring the learner to know $\pi_{\D}$ in advance. Ignoring the dependence on $\omega$ and $\pi_{\D}$, our bounds incur only a multiplicative $\theta$ gap compared to established label complexity bounds for the classical CAL algorithm in standard active learning \cite{cohn1994improving, hanneke2009theoretical}.

Our analysis is based on the assumption that the underlying data distribution is continuous. This is a natural assumption, as a substantial line of work in active learning studies settings where the instance distribution admits a density and the decision boundary (or regression function) satisfies smoothness or regularity conditions; under such assumptions, one can obtain refined label complexity guarantees via geometric and disagreement-based arguments \cite{castro2007minimax, wang2011smoothness, locatelli2018adaptive, kpotufe2022margin}.


\subsection{Related Work} \label{sec:intro-relared}

While active learning has been extensively studied in the fully supervised setting, it has received only limited attention in the context of PU learning. 

Notably, active PU learning (also referred to as active one-class classification) has been studied using a variety of uncertainty sampling approaches, including Query-by-Committee, least confidence, margin-based, and entropy-based sampling methods \cite{abe2006outlier,gornitz2009active,6137386,schlachter2018active}. Prior work has also considered strategies that prioritize queries that are most likely to be positive \cite{he2006generalized,ghasemi2011active}. More recently, \citet{perini2020class} study active methods for empirical estimation of the class prior (i.e., the prevalence of positive examples) for PU learning.

However, with the exception of \citep{perini2020class}, these works are purely empirical, and even \citet{perini2020class} do not provide finite label complexity guarantees. In this context, our work is the first study establishing formal guarantees on the label complexity of active PU learning.


\section{Preliminaries} \label{sec:prelim}
\subsection{Active Learning}

Let $\X$ be a \emph{domain set}. The \emph{hypothesis class} $\cH$ is a set of functions $h: \X \rightarrow \{0, 1\}$. By $\D$, we denote a distribution over $\X$ called the data generating distribution. We use $\ell: \X \to \{0, 1\}$ to denote an underlying labeling rule.
In the study of active learning, there is a pool of data $\{x_1, x_2, \ldots \}$ sampled i.i.d.\ from $\D$. An active learner, at each time $t$, requests the label of an instance $x \in \X$ from the pool and receives $\ell(x)$. The goal of the learner is to output a function $f: \X \to \{0, 1\}$ that minimizes
$$
\err_{\D}(f, \ell) = \mathbb{E}_{x \sim \D} [f(x) \neq \ell(x)].
$$


For any multiset $S \subseteq \X$ define
$$
\hat \err_S(f, \ell) := \frac{1}{|S|} \sum_{x \in S} \mathbbm{1} \{ f(x) \neq \ell(x) \}\,.
$$
Also, define $\hat \Pr_S$ as the empirical distribution induced by $S$, i.e., for any event $A\subseteq\X$, $\Pr_S [A]$ is defined as the fraction of elements in the multiset $S$ that also belong to $A$. Formally,
$$\hat \Pr_S [A] := \frac{|\{x\in S: x\in A\}|}{|S|}\,,
$$
where both numerator and denominator account for elements in $S$ with their multiplicity in $S$.

\begin{definition}
    Define the pseudo metrics $\rho_\D, \rho_S: 2^\X \times 2^\X \rightarrow \Re^{\geq 0}$ over $2^{\X}$ by
    $$
    \begin{aligned}
        \forall f, g \in 2^\X, \; \rho_\D(f, g) & :=  \Pr_{x \sim \D}\left[f(x) \neq g(x) \right] \,,\\
        \forall f, g \in 2^\X, \; \rho_S(f, g) & := \frac{|\{x: f(x) \neq g(x)\}|}{{|S|}}\,.
    \end{aligned}
    $$
\end{definition}

\begin{definition} [Region of Disagreement \cite{hanneke2007bound}]
    The region of disagreement of any set $V \subseteq \cH$ is defined as
$$\DIS(V)=\left\{x \in \X: \exists h_1, h_2 \in V \text { s.t. } h_1(x) \neq h_2(x)\right\}\,.$$
\end{definition}

\begin{definition} [Disagreement Rate \cite{hanneke2007bound}]
    The \emph{disagreement rate}, and the \emph{empirical disagreement rate}, respectively, of any set $V \subseteq \cH$ are defined as
    $$
    \begin{aligned}
        \Delta_\D (V) & := \Pr_{x \in \D}[x \in \DIS(V)]\,, \\
    \hat \Delta_S(V) & := \hat \Pr_S[\DIS(V)]\,.
    \end{aligned}
    $$
\end{definition}

\begin{definition}
    Let $f:\X \rightarrow \{0, 1\}$ and $r > 0$.
    The \emph{$r$-ball around $f$}
    with respect to $(\D, \cH)$ is defined
    as
  $$\mathcal B_{\D}(f, r)=\left\{h \in \cH: \rho_\D(f, h) \leq r\right\}\,.$$ 
    For a multiset $S \subseteq \X$, define the \emph{empirical $r$-ball} centered around $f$ with respect to $(S, \cH)$ as
    $$
  \hat{\mathcal B}_S(f, r)=\left\{h \in \cH: \rho_S(f, h) \leq r\right\}\,.
  $$
\end{definition}

\begin{definition} [Disagreement Coefficient \cite{hanneke2007bound}]
    The disagreement coefficient 
    is defined as 
    $$\theta =\sup _{h \in \cH, r > 0} \frac{\Delta_\D(\mathcal{B}_{\D} (h, r))}{r}\,.$$
\end{definition}

\subsection{Passive PU Learning}
The study of learning from positive and unlabeled examples (PU learning) considers the setting where the learner has access only to positive examples and unlabeled data. In the following, we state an existing result from passive PU learning due to \citet{liu2002partially}, which we will use later.

\begin{theorem}
    [Theorem 1 of \cite{liu2002partially}]\label{thm:real-passive-PU}
    Let $\mathcal{H}$ be a hypothesis class of VC dimension $d$ over the domain $\X$. Then there exists a constant $M_1> 1$ such that for any $\varepsilon,\delta>0$ and any distribution $\D$ over $\X\times\{0,1\}$ that is realized by $\mathcal{H}$, if
\[
k \;\ge\; M_1 \frac{d\ln(1/\varepsilon)+\ln(1/\delta)}{\varepsilon},
\]
then the following holds. Let $S^U$ be an unlabeled sample of size $k$ drawn i.i.d.\ from $\D$, and let $S^P$ be a positive sample of size $k$ drawn i.i.d.\ from $\D(\cdot \mid y=1)$. Define
\[
\begin{aligned}
    & h^{PU}\;:=\;\arg\min_{h\in\mathcal{H}(S^P)} \bigl|h\cap S^U\bigr| \\
& \text{s.t.} \quad \mathcal{H}(S^P)\;:=\;\{h\in\mathcal{H}:\ \forall x\in S^P,\ h(x)=1\},
\end{aligned}
\qquad
\]
Then with probability at least $1-\delta$, we have $\err_{\D}(h^{PU},\ell)\le \varepsilon$.

\end{theorem}

\section{Setup} \label{sec:setup}
Fix $\omega \in (0, 1]$ and define a function $s: \X \to \{0, 1\}$ as follows. Suppose there is a coin that flips heads with probability $\omega$ and tails with probability $1 - \omega$. For every $x \in \X$, if $\ell(x) = 0$ then $s(x) = 0$. Otherwise, i.e., if $\ell(x) = 1$, the coin is tossed; if it lands on heads, then $s(x) = 1$, else $s(x) = 0$. We assume an active PU learner which, at each time $t \in [n]$, can query a datum $x$ from the pool of unlabeled data. If $s(x) = 1$, the learner receives feedback $f := \ell(x)$; otherwise, the learner receives feedback $f := \star$.



Moreover, denote
\(
\pi_{\D} := \D\!\big(\{x \in \X : \ell(x)=1\}\big),
\)
and for any multiset $S \subseteq \X$ define
\(
\hat{\pi}_S := \frac{\lvert\{x \in S : \ell(x)=1\}\rvert}{\lvert S\rvert}.
\)
Furthermore, define the positive disagreement region of $\cH$ as
\(
\DIS^P(\cH) := \{x \in \DIS(\cH) : \ell(x)=1\}.
\)
Throughout this paper, we focus on the realizable setting; that is, we
assume $\ell \in \cH$. We also assume that the VC dimension of $\cH$ is
$\VCD(\cH)=d$.

\begin{assumption} \label{asmp:cont}
    $\D$ is a continuous distribution, i.e., for every $x \in \X$ we have $\D(x) = 0$.
\end{assumption}

\begin{remark} \label{remark:cont-asmp}
     Let $\D$ be a distribution that satisfies Assumption~\ref{asmp:cont}. Let $S$ be an i.i.d.\ sample of $\D(\cdot \mid x \in A)$ for $A \subseteq \X$. Then $S' = \{x \in S \mid s(x) = 1\}$ is an i.i.d.\ sample of $\D(\cdot \mid x \in A, \ell(x) = 1)$.
\end{remark}

    
    

\section{Analysis with Known $\pi_\D$} \label{sec:known-pi}
In this section, we study active PU learning for the case when $\pi_{\D}$ is known. For this setting, we propose Algorithm~\ref{alg:real-active-pu-known-pi}, which is simply the classical CAL algorithm of \citet{cohn1994improving} applied to the restricted hypothesis class consisting of all $h \in \cH$ whose predicted positive mass over an unlabeled sample is approximately at most $\pi_{\D}$.


In Theorem~\ref{thm:active-pu-known-pi} we analyze the label complexity of this setup. The proof is inspired by Hanneke's \citeyearpar{hanneke2009theoretical} analysis of the CAL algorithm for active learning. The key ideas behind the analysis are as follows: (i) The initial pruning of the hypothesis class ensures that, for every $h$,
\begin{align*}
    &\Pr[\ell(x) = 0, h(x) = 1] - \Pr[\ell(x) = 1, h(x) = 0] \\
    &\quad = \Pr[h(x) = 1] - \Pr[\ell(x) = 1] \lessapprox 0.
\end{align*}
This means that the false positive rate can always be bounded from above by the false negative rate. Thus, using positively labeled data, we can bound the false negative rate with standard PAC bounds. Consequently, also the total error can be bounded. (ii) Using this with an argument similar to Hanneke's \citeyearpar{hanneke2009theoretical}, we can show that after every $\tilde O(d \theta )$ instances with labels revealed, the disagreement rate $\Delta_\D(V_t)$ can be cut in half with high probability. (iii) Using Lemma~\ref{lem:pos-rate-known}, we guarantee that the proportion of positively labeled points in $\DIS(V_t)$ has a lower bound of $\Omega(1/\theta)$ 
for all $t$. This ensures that the query response rate remains at least 
$\Omega(\omega/\theta)$.





\begin{algorithm}[tb!]
\caption{Active PU learning with known $\pi_\D$}
\label{alg:real-active-pu-known-pi}
\begin{algorithmic}
\STATE {\bfseries Input:} Hypothesis class $\mathcal{H}$,
accuracy parameter $\varepsilon \in (0, 1]$, confidence parameter $\delta \in (0, 1]$.

\STATE $t \leftarrow 0, \gamma \leftarrow \frac{\varepsilon}{8 \theta}, k \leftarrow 128 \frac{d \ln(128 / \gamma) + \ln(8 / \delta)}{\gamma^2}$
\STATE $S \leftarrow$ Sample $k$ instances from $\D$
\STATE $V_0 \leftarrow \{ h \in \mathcal{H}: \hat{\Pr}_S[h] \leq \pi_\D + \gamma \}$  \label{line:known-pi-V-init}
\FOR{$m = 1, 2, \ldots$}
    \STATE Draw a sample $x_m$ from $\D$.
    \IF{$x_m \in \DIS(V_t)$ }
    \STATE Request label of $x_m$.
    \STATE $t \leftarrow t+1$, $V_t \leftarrow V_{t - 1}$
    \IF{label of $x_m$ is not $\star$}
        \STATE $V_t \leftarrow \{h \in V_t: h(x_m) = 1\}$
    \ENDIF
    \IF{$\Delta_\D(V_t) \leq \varepsilon$}
        \STATE {\bfseries return} any $\hat h \in V_t$.
    \ENDIF
    \ENDIF      
\ENDFOR
\end{algorithmic}
    
\end{algorithm}





\begin{lemma}[Multiplicative Chernoff bounds \cite{motwani1996randomized}] \label{lem:mult-chern-bound} Let $X_1, \ldots, X_m$ be independent random variables drawn according to some distribution $\mathcal{D}$ with mean $p$ and support included in $[0,1]$. Then, for any $\gamma \in\left[0, \frac{1}{p}-1\right]$, the following inequalities hold for $\widehat{p}=\frac{1}{m} \sum_{i=1}^m X_i$:
$$
\begin{aligned}
& \mathbb{P}[\widehat{p} \geq(1+\gamma) p] \leq e^{-\frac{m p \gamma^2}{3}} \,,\\
& \mathbb{P}[\widehat{p} \leq(1-\gamma) p] \leq e^{-\frac{m p \gamma^2}{2}}\,.
\end{aligned}
$$
\end{lemma}

\begin{lemma} \label{lem:pos-rate-known}
    Let $\varepsilon > 0$. Consider any $V \subseteq \mathcal{B}( \mathbf{0}, \pi_\D + \frac{\varepsilon}{4})$ such that $\ell \in V$ and $\sup_{h \in \cH} \err_\D(h, \ell) > \varepsilon$. Then
    $$
    \Pr[\ell(x) = 1 \mid x \in \DIS(V)] \geq \frac{1}{4\theta}\,.
    $$
    
\end{lemma}
\begin{proof}
Denote $b := \Pr_{x \sim \D} [\forall h \in V: h(x) = 1]$. By definition we have
\begin{equation} \label{eq:pos-rate}
    \Pr[\ell(x) = 1,\, x \in \DIS(V)] \; = \; \pi_\D - b\,.
\end{equation}
Denote $f = \bigcap_{h \in V} h$. Since $V \subseteq \mathcal{B}_\D(\mathbf{0}, \pi_\D + \frac{\varepsilon}{4})$, we also have $V \subseteq \mathcal{B}_\D(f,\, \pi_\D - b + \frac{\varepsilon}{4})$. Since $\ell \in V$, we obtain $V \subseteq \mathcal{B}_\D\left(\ell,\, 2 \left(\pi_\D - b + \frac{\varepsilon}{4}\right)\right)$.
Thus,
\[
    2 \left(\pi_\D - b + \frac{\varepsilon}{4}\right) \;\geq\; \sup_{h \in V} \err_\D(h, \ell) \;>\; \varepsilon\,,
\]
which implies $\pi_\D - b > \frac{\varepsilon}{4}$.  This yields
\(
    V \subseteq \mathcal{B}(\ell,\, 4(\pi_\D - b)).
\)
Therefore,
\begin{equation} 
    \Delta_\D(V) 
        \;\leq\; \Delta_\D\left(\mathcal{B}_\D  (\ell, 4(\pi_\D - b))\right) 
        \;\leq\; 4\theta (\pi_\D - b)\,.
\end{equation}
Combining this with \eqref{eq:pos-rate} completes the proof.
\end{proof}

\begin{theorem} \label{thm:active-pu-known-pi}
     For any $(\varepsilon, \delta) \in (0,1] \times (0,1]$, under Assumption~\ref{asmp:cont}, given inputs $\cH$, $\varepsilon$, and $\delta$, with probability at least $1-\delta$, Algorithm~\ref{alg:real-active-pu-known-pi} outputs a hypothesis $\hat h$ satisfying $\err_{\D}(\hat h,\ell)\le \varepsilon$. Moreover, the number of label requests made by the algorithm is at most
    \[
    O\!\left(\frac{\ln(1/\varepsilon)\,\theta^2\bigl(d\ln(\theta)+\ln\ln(1/\varepsilon)+\ln(1/\delta)\bigr)}{\omega}\right).
    \]

\end{theorem}

\begin{proof}
    Let $x_{m_t}$ denote the example corresponding to the $t^{\text{th}}$ label request.  
    By standard passive learning literature (see, e.g. \cite{haussler1992decision, vapnik2006estimation}), 
    for all $h \in V_0$, with probability at least $1 - \delta/2$, it holds that
    \begin{equation} \label{eq:thm-pu-known-pi-0}
        \big|\hat{\Pr}_S[h(x) = 1] - \Pr_{x \sim \D} [h(x) = 1]\big| < \gamma.
    \end{equation}

    Since $V_0 = \hat{\mathcal{B}}_S(\mathbf{0}, \pi_\D + \gamma)$, we obtain
    \[
        \mathcal{B}_{\D} (\mathbf{0}, \pi_\D) \subseteq V_0 \subseteq \mathcal{B}_{\D} (\mathbf{0}, \pi_\D + 2\gamma).
    \]
    Therefore, for every $h \in V_0$,
    \begin{equation} \label{eq:thm-pu-known-pi-1}
    \begin{aligned}
        & \Pr[\ell(x) = 0, h(x) = 1] - \Pr[\ell(x) = 1, h(x) = 0]  \\
        &\quad = \Pr[\ell(x) = 0, h(x) = 1] \\
        & \quad \quad 
        - \big(\Pr[\ell(x) = 1] - \Pr[\ell(x) = 1, h(x) = 1]\big) \\
        & \quad = \Pr[\ell(x) = 0, h(x) = 1] \\&\quad\quad+ \Pr[\ell(x) = 1, h(x) = 1] - \pi_\D \\
        & \quad = \Pr[h(x) = 1] - \pi_\D \leq \pi_\D + 2\gamma - \pi_\D = 2\gamma.
    \end{aligned}
    \end{equation}
    Finally, since $\mathcal{B}(\mathbf{0}, \pi_\D) \subseteq V_0$, we conclude that $\ell \in V_0$. It is also immediate that $\ell$ will remain in $V_t$ for all $t$. Since $\gamma \leq \frac{\varepsilon}{8}$, we can use Lemma~\ref{lem:pos-rate-known} for $V_t$. Therefore, as long as $\sup_{h \in V_t} \err_\D(h, \ell) \geq \varepsilon$, we have
\begin{equation} \label{eq:thm-known-pi-2}
    \Pr[\ell(x) = 1 \mid x \in \DIS(V_t)] \;\geq\; \tfrac{1}{4 \theta}\,.
\end{equation}
Set $N = \log_2(1 / \varepsilon)$, and 
\[
    \lambda \;=\; \frac{ 128 \theta^2 (4d ln(192 \theta ) + ln(8N / \delta)) }{\omega}.
\]

\paragraph{Claim 1.} 
For any $i \in [N]$ and time step $t = (i - 1)\lambda + 1$ such that $\Delta_\D(V_t) \geq \varepsilon$, with probability at least $1 - \delta / 2N$ we have
\[
    \Delta_\D(V_{t'}) \;\leq\; \tfrac{1}{2}\Delta_\D(V_t),
    \quad \text{where } t' = t + \lambda\,.
\]

\noindent\textbf{Proof of Claim 1.}  
Define 
\[
\begin{aligned}
    A & := \{x_m : m_t < m \leq m_{t'},\, x_m \in \DIS(V_t)\}, 
    \\
    B & := \{x \in A : s(x) = 1\}.
\end{aligned}
\]
Clearly $|A| \geq \lambda$. Let $\lambda' := \tfrac{\lambda \omega}{8 \theta}$.  
By \eqref{eq:thm-known-pi-2} and the multiplicative Chernoff bound (Lemma~\ref{lem:mult-chern-bound}), we obtain
\[
    \Pr \big[\,|B| \leq \lambda' \,\big] 
        \leq \exp\!\left(- \tfrac{\lambda'}{4}\right) 
        < \tfrac{\delta}{4 N}.
\]
Note that $A$ is an i.i.d.\ sample from $\D(\cdot \mid \DIS(V_t))$. Due to Remark~\ref{remark:cont-asmp}, this means that $B$ is an i.i.d.\ sample from $\D(\cdot \mid \DIS^P(V_t))$.
Moreover,
\[
    V_{t'} \;=\; \{h \in V_t : h(x) = 1 \;\;\mbox{for all }x \in B\}.
\]
 By standard results from the passive learning literature (see, e.g., \cite{blumer1989learnability, vapnik2006estimation}), with probability at least $1 - \delta / 4 N$, for all $h' \in V_{t'}$ we have
\begin{equation} \label{eq:thm-known-pi-5}
    \err_{\D(. \mid x \in \DIS^P(V_t))}(h', \ell) \;\leq\; \tfrac{1}{8 \theta}.
\end{equation}

Consequently,
\[
\begin{aligned}
    \Pr[h'(x) = 0,\, \ell(x) = 1] 
        &= \Pr[h'(x) \neq \ell(x),\, x \in \DIS^P(V_t)] \\
        &\leq \frac{\Pr[x \in \DIS^P(V_t)]}{8 \theta} \leq \frac{\Delta_\D(V_t)}{8 \theta}.
\end{aligned}
\]

Using \eqref{eq:thm-pu-known-pi-1}, it follows that
\[
    \Pr[h'(x) = 1,\, \ell(x) = 0] 
        \;\leq\; \tfrac{\Delta_\D(V_t)}{8 \theta} + 2\gamma.
\]
Since $\Delta_\D(V_t) \geq \varepsilon$, we conclude
\[
\begin{aligned}
    \err_\D(h', \ell) 
        &= \Pr[h'(x) = 0,\, \ell(x) = 1] \\&\quad + \Pr[h'(x) = 1,\, \ell(x) = 0] \\
        &\leq \tfrac{\Delta_\D(V_t)}{4 \theta} + 2\gamma \leq \tfrac{\Delta_\D(V_t)}{2 \theta}.
\end{aligned}
\]

Thus, $V_{t'} \;\subseteq\; \mathcal{B}\!\left(\ell,\, \tfrac{\Delta_\D(V_t)}{2\theta}\right)$,
and therefore
\[
    \Delta_\D(V_{t'}) 
        \;\leq\; \Pr\!\left(\DIS\!\Big(\mathcal{B}\big(\ell, \tfrac{\Delta_\D(V_t)}{2\theta}\big)\Big)\right) 
        \;\leq\; \tfrac{1}{2}\Delta_\D(V_t).
\]
This completes the proof of the Claim 1.

Note that the number of time steps considered in Claim 1 is at most $N$.  
Hence, by a union bound, Claim 1 holds for all such steps with probability at least $1 - \delta / 2$.  
Thus, by a union bound with the event in \eqref{eq:thm-pu-known-pi-0} holding, after
$t = \lambda \cdot \log_2(1/\varepsilon)$ queries, with probability at least $1 - \delta$,
\[
    \sup_{\hat h \in V_t} \err_\D(\hat h, \ell) \leq \Delta_\D(V_t) \leq \varepsilon.
\]
This completes the proof.
\end{proof}

\section{Analysis with Unknown $\pi_\D$}\label{sec:unknown-pi}

In this section, we study active PU learning in the more general setting where $\pi_\D$ is unknown. For this case, we propose Algorithm~\ref{alg:real-active-pu-unknown-pi}, with the formal analysis presented in Theorem~\ref{thm:active-PU-unkown}.

When $\pi_\D$ is unknown, the learner is unable to restrict the version space to hypotheses whose predicted positive rate over a pool of unlabeled data is close to $\pi_\D$. This increases the difficulty of the problem, as in the previous section we were only able to control the response rate of queries in Lemma~\ref{lem:pos-rate-known} by exploiting this property. The key idea in our analysis is to employ a binary-search--style procedure for estimating $\pi_\D$ and restricting the version space accordingly.

We first obtain an estimate $\hat{\omega}$ of $\omega$, using Algorithm~\ref{alg:est-rate}. By Theorem~\ref{thm:estrate}, this estimate is 
within a constant factor of the true value $\omega$. We then sample a large unlabeled set $S_1$. For each iteration $i$, let $u_i$ and $b_i$ denote upper and lower bounds, respectively, on $\hat{\pi}_{S_1}$, the empirical positive rate in $S_1$. We initialize $u_0 = 1$. At iteration $i$, we restrict the version space to hypotheses whose predicted positive rate over $S_1$ is at most $u_i$, and we set $b_i$ to be the fraction of samples in $S_1$ that are predicted positive by all $h \in V_i$.

In Lemma~\ref{lem:u-binary-search}, we show that whenever the response rate of queries from $\DIS(V_i)$ is less than $O(\hat{\omega} / \theta)$, it must hold that
\(
\hat{\pi}_{S_1} < \tfrac{u_i + b_i}{2}.
\)
In this case, we update the upper bound by $u_{i+1} \gets (u_i + b_i)/2$.

The empirical error (which is close to the true error) of each classifier is always bounded from above by $2(u_i - b_i)$. Consequently, such a halving can occur at most $O(\log_2(1/\varepsilon))$ times. After this point, for all remaining queries we are guaranteed a response rate of at least $\Omega(\omega / \theta)$. From then on, the analysis proceeds analogously to the case where $\pi_\D$ is known. In particular, using $\tilde{O}(d \theta)$ queries drawn from $\D$ and restricted to instances $x \in \DIS(V_i)$, we can reduce $\Delta_\D(V_i)$ by a factor of two.

\begin{algorithm}[tb!]
\caption{$\operatorname{EstRate}(\cH, \delta)$}
\label{alg:est-rate}
\begin{algorithmic}[1]
\STATE {\bfseries Input:} Hypothesis class $\mathcal{H}$, confidence parameter $\delta$
\STATE $P \leftarrow \varnothing$; $i, r \leftarrow 0$.
\WHILE{ $r < 8 \ln (8 / \delta)$}
    \STATE $U \leftarrow$ Sample $2^i$ instances from $\D$
    \STATE Request label of each instance in $U$ and let $R$ be the set of instances with feedback $1$. 
    \STATELABEL{line:estrate-V} \STATE  $V \leftarrow \{h \in \cH: \forall x \in R: x \in h \}$
    \STATELABEL{line:sample-S} \STATE $S \leftarrow$ Sample $2^i$ instances from $\D$
    \STATE \textcolor{blue}{\footnotesize\textit{/* Pruning the version space using existing passive PU results */}}
    \STATELABEL{line:estrate-h} \STATE $h^* \leftarrow \argmin_{h \in V} |h \cap S|$ 
    \STATELABEL{line:estrate-gamma} \STATE $\gamma \leftarrow \frac{(M_1+M_2)\bigl(d\ln(2|R|/d)+\ln(4/\delta)\bigr)}{|R|}$
    \STATE $V \leftarrow V \cap \hat{\mathcal{B}}_S \left(h^*, 3 \gamma \right)$
    
    \STATE \textcolor{blue}{\footnotesize\textit{/* Finding samples that are guaranteed to have positive label */}}
    \STATE $P \leftarrow \left \{x \in S \mid \forall h' \in V: h'(x) = 1 \right \}$
    \STATE Request label of each instance in $P$ and let $r$ be the number of instances with feedback  $1$.
    \STATE $i \leftarrow i + 1$
\ENDWHILE
 \STATE {\bfseries Return:} $\frac{r}{|P|}$ 
\end{algorithmic}
\end{algorithm}

\begin{algorithm}[tb!]
\caption{Active PU learning with unknown $\pi_\D$}
\label{alg:real-active-pu-unknown-pi}
\begin{algorithmic}[1]
\STATE {\bfseries Input:} Hypothesis class $\mathcal{H}$, accuracy parameter $\varepsilon$, confidence parameter $\delta$.
\STATE $\hat \omega  \leftarrow \operatorname{EstRate}(\mathcal{H}, \delta / 8)$
\STATE $i \leftarrow 0, u_0 \leftarrow 1, V_0 \leftarrow \cH, N \leftarrow 2\log_2(6 / \varepsilon)$
\STATE $\lambda_1 \leftarrow \frac{768 \theta \ln(8N / \delta)}{\hat \omega}$
\STATE $\lambda_2 \leftarrow \frac{M_3 \theta^2   \left( d \ln(\theta) + \ln(N / \delta)\right)}{\hat \omega}$
\STATE $S_1 \leftarrow$ Sample $\frac{(3M_2+96)\bigl(d\ln(3/\varepsilon)+\ln(16N/\delta)\bigr)}{\varepsilon}$ instances from $\D$
\WHILE{$u_i - b_i > \frac{\varepsilon}{6}$ and $\hat \Delta_{S_1}(V_i) > \frac{\varepsilon}{3}$}
        \STATE $U \leftarrow$ Rejection sample $\lambda_1$ samples $x$ from $\D$ satisfying $x \in \DIS(V_i)$
        \STATE Request label of instances in $U$ and let $R_1$ be the set of instances with feedback $1$.
        \STATELABEL{line:active-PU-if-u-small} \IF{ $\frac{|R_1|}{\lambda_1} < \frac{\hat \omega}{128 \theta}$} 
            \STATE \textcolor{blue}{{\footnotesize\textit{/* The feedback rate is small: update $u$ */}}}
            \STATE $b_i \leftarrow \hat Pr_{S_1} [\forall h \in V_i: h(x) = 1]$. \;
            \STATE $u_{i + 1} \leftarrow \frac{u_i + b_{i}}{2}$.
            \STATELABEL{line:active-PU-V-cap-B} \STATE $V_{i + 1} \leftarrow  \{h \in V_i: \forall x\in R_1: h(x) = 1 \} \cap \hat{\mathcal{B}}_{S_1} (\mathbf 0, u_{i + 1})$  
        \STATELABEL{line:active-PU-else} \ELSE 
            \STATE \textcolor{blue}{{\footnotesize\textit{/* The feedback rate is large: prune the version space using existing passive PU results */}}}
            \STATE $U \leftarrow$ Rejection sample $\lambda_2$ samples $x$ from $\D$ satisfying $x \in \DIS(V_i)$
            \STATE Request label of instances in $U$ and let $R_2$ be the set of instances with feedback $1$.
            \STATE $S_2 \leftarrow$ Rejection sample $|R_2|$ samples $x$ from $\D$ satisfying $x \in \DIS(V_i)$.
            \STATELABEL{line:active-PU-V-R-update} \STATE   $V_{i + 1} \leftarrow \{h \in V_i\mid \forall x \in R_2: h(x) = 1 \}$ 
            \STATELABEL{line:active-PU-h-update} \STATE $h_{i + 1} \leftarrow \argmin_{h \in V_{i + 1}} |h \cap S_2|$
            \STATE $\gamma_{i+1} \leftarrow \frac{(2M_1+M_2)\bigl(d\ln(2|R_2|/d)+\ln(8N/\delta)\bigr)}{|R_2|}$
            \STATELABEL{line:active-PU-V-update} \STATE  $V_{i + 1} \leftarrow V_{i + 1} \cap \hat{\mathcal{B}}_{S_2 } (h_{i + 1},  \gamma_{i + 1})$
        \ENDIF
        \STATE $i \leftarrow i + 1$
\ENDWHILE
\STATE {\bfseries Return:} $h_i$ 
\end{algorithmic}
\end{algorithm}

\begin{theorem}[\cite{haussler1992decision}]
    Suppose that $\VCD(\cH) = d$ and let $f : \X \to \{0,1\}$. Then, for all $\alpha, \nu > 0$, a sample $S$ of size $n$ drawn i.i.d.\ from $\D$ satisfies
    $$
    \begin{aligned}
        \Pr \Big[ \exists h \in \cH: \left|\rho_\D(h, f) - \rho_S(h, f)\right|\\
         > \alpha \left( \rho_\D(h, f) - \rho_S(h, f) + \nu \right) \Big] \\
        \leq 8 \left( \frac{16 e}{\alpha \nu} \ln \frac{16 e}{\alpha \nu} \right)^d e^{\alpha^2 \nu n / 8}\,.
    \end{aligned}
    $$
\end{theorem}
Now set $\alpha = 1/3$ and $\nu = 2\varepsilon$ for hypothesis class $\cH \Delta \cH := \{ h \oplus h' \mid h, h' \in \cH\}$, as well as $f := \mathbf 0$. Since $\VCD(\cH \Delta \cH) \leq 2 \VCD(\cH) + 1$ \cite{ben1998combinatorial}, one obtains the following result.

\begin{corollary}
    \label{cor:B-hat-B}
    There exists a constant $M_2 > 1$ such that for all $\delta > 0$, with probability $1 - \delta$ for a sample $S$
    i.i.d. sampled from $\D$ we have (i) for all $h, h' \in \cH$
    $$
    \begin{aligned}
         \rho_{\D}(h, h') & \leq 2  \rho_{S} (h, h') + \varepsilon, \; \text{and} \\
          \rho_{S}(h, h') & \leq 2  \rho_{\D} (h, h') + \varepsilon\,;
    \end{aligned}
    $$
    (ii) for all $r > 0$ and all $f \in \cH$ 
    \begin{equation}
        \begin{aligned} \label{eq:B-hat-B}
         \mathcal B_\D(f, r) & \subseteq \hat{\mathcal{B}}_S \left(f, 2r + \varepsilon\right),  \; \text{and} \\
          \hat{\mathcal B}_{S} (f, r) &\subseteq \mathcal{B}_\D (f, 2r + \varepsilon)\,,
    \end{aligned}
    \end{equation}
    where $\varepsilon =  M_2 \frac{d \ln(2|S| / d) + \ln(1 / \delta)}{|S|}$.
\end{corollary} 

From this point forward, $M_1$ and $M_2$ refers to the constants given in Theorem~\ref{thm:real-passive-PU} and Corollary~\ref{cor:B-hat-B}, respectively.


\begin{restatable}{theorem}{EstRate} \label{thm:estrate}
    For any $\delta>0$, under Assumption~\ref{asmp:cont}, with probability at least $1-\delta$ we have
\[
\frac{\omega}{2}\ \le\ \operatorname{EstRate}(\cH,\delta)\ \le\ 2\omega .
\]
Moreover, the number of label requests made by $\operatorname{EstRate}$ is
\[
O\!\left(\frac{\theta\bigl(d \ln(\theta / \pi_\D)+\ln(1/\delta)\bigr)}{\pi_D^{2}\,\omega}\right).
\]
\end{restatable}

The proof of Theorem~\ref{thm:estrate} is deferred to the Appendix.

\begin{lemma} \label{lem:u-binary-search}
   
    Fix an iteration $i$ in Algorithm~\ref{alg:real-active-pu-unknown-pi}. Suppose $\ell \in V_i$, $u_i - b_i \geq \frac{\varepsilon}{6}$ and $\hat{\pi}_{S_1} \leq \frac{b_i + u_i}{2}$.
    Then with probability at least $1 - \frac{\delta}{8N}$ we have
    $$
    \Pr(\ell(x) = 1 \mid \DIS(V_i)) < \frac{1}{32 \theta}.
    $$
\end{lemma}
\begin{proof}
    Since $\ell \in V_i$, by definition of $b_i$ we have
    \begin{equation*} 
        \hat{\Pr}_{S_1} (x \in \DIS(V_i) \text{ and } \ell(x) = 1) = (\hat{\pi}_{S_1} - b_i)\,.
    \end{equation*}
    Since 
    $$\hat{\pi}_{S_1} - b_i \geq \frac{u_i - b_i}{2} \geq \frac{\varepsilon}{12},$$
    applying the multiplicative Chernoff bound (Lemma~\ref{lem:mult-chern-bound}), we obtain
     \begin{equation} \label{eq:cut-in-half-1}
    \begin{aligned}
        \Pr \left[\Pr(x \in \DIS(V_i) \text{ and } \ell(x) = 1) <  \frac{(\hat{\pi}_{S_1} - b_i)}{2}\right] \\ < \exp \left(\frac{-|S_1| \varepsilon}{12 * 8} \right) \leq \frac{\delta}{16N}
    \end{aligned}
    \end{equation}

    Denote $f_i(x) := \mathbbm{1}\{ \forall h' \in V_i: h'(x) = 1\}$. Then by definition of $V_i$ we have $V_i \subseteq \hat{\mathcal{B}}_{S_1}(\mathbf{0}, u_i)$. Since $b_i$ is the ratio of instances in $S_1$ where all of $V_i$ have label 1, we have 
    $$V_i \subseteq \hat{\mathcal{B}}_{S_1}(f_i, u_i - b_i) \subseteq \hat{\mathcal{B}}_{S_1}(\ell, 2(u_i - b_i)).$$

    Note that $S_1$ satisfies
    $$
    M_2 \frac{d \ln(2|S_1| / d) + \ln(16N / \delta)}{|S_1|} \leq \frac{2\varepsilon}{3}\,.
    $$
    Thus, according to Corollary~\ref{cor:B-hat-B} (ii), with probability $1 - \delta / 2$, we have \begin{equation} \label{eq:cut-in-half-2}
        \hat{\mathcal{B}}_{S_1}(\ell, 2(u_i - b_i)) \subseteq \mathcal{B}_\D(\ell, 8(u_i - b_i))\,.
    \end{equation}
    Thus,
    \begin{equation*} 
    \begin{aligned}
        \Delta_\D(V_i) & \leq \Delta_\D\left(\mathcal B_\D (\ell, 8(u_i - b_i))\right) \\
        & \leq  8 \theta (u_i - b_i) \leq 16 \theta (\hat{\pi}_{S_1} - b_i)\,.
    \end{aligned}
    \end{equation*}
    
    Combining this with \eqref{eq:cut-in-half-1} completes the proof.
\end{proof}

\begin{theorem} \label{thm:active-PU-unkown}
    There exist a constant $M_3 > 1$ such that for any
$(\varepsilon, \delta) \in (0,1] \times (0,1]$, under
Assumption~\ref{asmp:cont}, given $\cH$, $\varepsilon$, and $\delta$,
Algorithm~\ref{alg:real-active-pu-unknown-pi} 
makes 
\[
\begin{aligned}
O\Bigg(
\frac{\ln (1 / \varepsilon)\,\theta^2\,
\big(d \ln(\theta) + \ln \ln(1 / \varepsilon) + \ln(1 / \delta)\big)}{\omega} \\
  + \frac{\theta\big(d \ln(\theta / \pi_\D) + \ln(1 / \delta)\big)}{\pi_{\D}^2\,\omega}
\Bigg)
\end{aligned}
\]
label requests 
and then 
outputs a hypothesis
$\hat h$ which, with probability at least $1-\delta$, satisfies
\(
\err_{\D}(\hat h, \ell) \leq \varepsilon
\).


\end{theorem}
\begin{proof}
    Due to Theorem~\ref{thm:estrate}, we can assume that $\hat \omega \notin \left[\frac{\omega}{2}, 2 \omega \right]$ with probability $1 - \delta / 8$. 
    For each iteration $i$,
    let $\D_i$ denote the distribution $\D$ conditioned on the disagreement region $\DIS(V_i)$, that is,
    \[
    \forall A \subseteq \X:\quad \D_i(A) := \D(x \in A \mid x \in \DIS(V_i)),
    \]
    and denote $p_i = \D_i(\ell(x) = 1)$. 
    We begin by establishing the following claims.

    \paragraph{Claim 1:} With probability at least $1 - \delta / 8$, for every iteration $i = 0, \ldots ,N - 1$:  \\
    (i) If $p_i \geq \frac{1}{32 \theta}$ , the algorithm enters the \texttt{if} statement in Line~\ref{line:active-PU-if-u-small}. \\
    (ii) If $p_i \leq \frac{1}{512 \theta}$ the algorithm does not enter the \texttt{if} Line~\ref{line:active-PU-if-u-small}.   \\
    \textbf{Proof.} 
    (i) Note that $\mathbb{E} \left[\frac{|R_1|}{\lambda_1}\right] = \omega p_i \geq \frac{\hat \omega }{64 \theta}$. Thus, using the multiplicative Chernoff bound (Lemma~\ref{lem:mult-chern-bound}), the probability of the algorithm not entering the \texttt{if} is at most
    \begin{equation*} 
        \Pr\left[ \frac{|R_1|}{\lambda_1} < \frac{\hat \omega}{128 \theta} \right] \leq \exp \left( - \frac{\lambda_1 \hat \omega }{8 * 64 \theta} \right) < \frac{\delta}{8N}.
    \end{equation*}
    
    (ii) Note that $\mathbb{E} \left[\frac{|R_1|}{\lambda_1} \right] = \omega p_i \geq \frac{\hat \omega }{256 \theta}$. Thus, the probability of the algorithm not entering the \texttt{if} is at most
    \begin{equation*} 
        \Pr\left[ \frac{|R_1|}{\lambda_1} <  \frac{\hat \omega}{128 \theta} \right] \leq \exp \left( - \frac{\lambda_1  \hat \omega }{3 * 256 \theta} \right) < \frac{\delta}{8N}.
    \end{equation*}
    Taking a union bound over all iterations $i = 0, \ldots, N-1$, completes the proof of Claim 1. 

   Note that by Corollary~\ref{cor:B-hat-B}(i), for every iteration $i = 0, \ldots, N-1$, with probability at least $1 - \delta/8$, for all $h, h' \in V_i$, it holds that
    \begin{equation} \label{eq:active-PU-0}
        \begin{aligned}
         \rho_{\D_i}(h, h') & \leq 2  \rho_{S_2} (h, h') + \varepsilon, \; \text{and} \\
          \rho_{S_2}(h, h') & \leq 2  \rho_{\D_i} (h, h') + \varepsilon\,,
    \end{aligned}
    \end{equation}
    where $\varepsilon =  M_2 \frac{d \ln(2|S_2| / d) + \ln(8 / \delta)}{|S_2|}$. 

    \paragraph{Claim 2:} With probability at least $1 - \delta/4$, for iteration $i = 0, \ldots, N - 1$, we have $\ell \in V_{i + 1}$. \\[0.3em]
\textbf{Proof.}  
It suffices to show that, with high probability, whenever $V_i$ is updated to $V_{i + 1}$, $\ell$ remains in $V_{i + 1}$. Note that $V_i$ is only updated in Lines~\ref{line:active-PU-V-R-update}, \ref{line:active-PU-V-update}, and \ref{line:active-PU-V-cap-B}. It is immediate that $\ell$ remains in $V_{i + 1}$ when the update occurs in Line~\ref{line:active-PU-V-R-update}. Thus, we only need to analyze the updates in the other two lines.

\begin{itemize}
    \item \textbf{Update in Line~\ref{line:active-PU-V-update}:}  
    Note that due to Remark~\ref{remark:cont-asmp},
    $R_2$ is an i.i.d. sample from $\DIS^P(V_i)$. Note that $h_{i+1}$ in Line~\ref{line:active-PU-h-update} corresponds to $h^{PU}$ in Theorem~\ref{thm:real-passive-PU}, where the underlying distribution is $\D_i$, the positive labeled sample is $R_2$, and the unlabeled sample is $S_2$.
    By Theorem~\ref{thm:real-passive-PU}, with probability $1 - \frac{\delta}{8N}$ we obtain
    \begin{equation} \label{eq:active-PU-claim-2-i}
    \begin{aligned}
        \err_{\D_i}(h_{i+1}, \ell) & \leq M_1 \frac{d \ln(2 |R_2| / d) + \ln(8 N / \delta)}{|R_2|} \\
    \end{aligned}
    \end{equation}
    Due to \eqref{eq:active-PU-0}, since $|S_2| = |R_2|$ we derive
    \[
    \begin{aligned}
        \rho_{S_2} (\ell, h_{i+ 1}) & \leq 2 \err_{\D_i}(h_{i + 1}, \ell) 
        \\ & \quad
         + M_2 \frac{d \ln(2|S_2| / d) 
         + \ln(8 / \delta)}{|S_2|}. \\
        & \leq  (2M_1 + M_2) \frac{d \ln(2|R_2| / d) + \ln(8N / \delta)}{|R_2|}.
        \end{aligned}
    \]
    Consequently, $\ell$ remains in $V_{i+1}$ during the update in Line~\ref{line:active-PU-V-update} with probability $1 - \frac{\delta}{8N}$.

    \item \textbf{Update in Line~\ref{line:active-PU-V-cap-B}:}  
    Since $\hat{\Pr}_{S_1}[\ell(x) = 1] = \hat{\pi}_{S_1}$, it suffices to show that, whenever $\hat{\pi}_{S_1} \leq \tfrac{u_i + b_i}{2}$, the algorithm does not enter the \texttt{if} in Line~\ref{line:active-PU-if-u-small}.  
    Suppose $\hat{\pi}_{S_1} - b_i \leq 2(u_i - b_i)$. 
    Since the algorithm entered the \texttt{while} loop, we have $u_i - b_i \geq \varepsilon / 4$. 
    Thus, by Lemma~\ref{lem:u-binary-search}, with probability $1 - \frac{\delta}{8N}$ we have $p_i \geq \tfrac{1}{32\theta}$. By Claim~1(i), this implies that the algorithm does not enter the \texttt{if} in Line~\ref{line:active-PU-if-u-small}. Therefore, $\ell$ remains in $V_{i + 1}$ with probability at least $1 - \frac{\delta}{8 N}$.
\end{itemize}

Each of the above updates can occur at most $N$ times. Taking a union bound over all events completes the proof of Claim~2. 


 \paragraph{Claim 3:} For any iteration $i \in \{ 0, \ldots, N\}$, if the update in Line~\ref{line:active-PU-V-update} happens with probability $1 - \frac{\delta}{8N}$, we have
\[
    \Delta_\D(V_{i+1}) \leq \tfrac{1}{2}\Delta_\D(V_{i}).
\]

\textbf{Proof.}  
By definition of $V_{i + 1}$, for every $h \in V_{i + 1}$ we have
\[
\begin{aligned}
    \rho_{S_2}(h_{i + 1}, h) & \leq (2M_1 + M_2)  \cdot \frac{d \ln(2|R_2| / d) + \ln(8N / \delta)}{|R_2|} \\
\end{aligned}
\]
Due to \eqref{eq:active-PU-0}
\[
\begin{aligned}
    & \rho_{\D_i}(h_{i + 1}, h) \leq (4M_1 + 3M_2) \frac{d \ln(2 |R_2| / d)  + \ln(8N/\delta)}{|R_2|}.
\end{aligned}
\]

Assume the above equation holds. Combining this with \eqref{eq:active-PU-claim-2-i} from Claim~2, we conclude that for all $h \in V_{i + 1}$,
\begin{equation} \label{eq:active-PU-claim-3-1}
    \begin{aligned}
    & \err_{\D_i} (h, \ell) \\
    & \quad \leq (5M_1 + 3M_2)  \frac{d \ln(2 |R_2| / d)  + \ln(8N/\delta)}{|R_2|}\,.
\end{aligned}
\end{equation}

    Next, notice that the algorithm only reaches Line~\ref{line:active-PU-V-update} when it entered the \texttt{else} in Line~\ref{line:active-PU-else}. Therefore, by Claim~1(ii), we have $p_i \geq \frac{1}{512 \theta}.$ Define $$\lambda_2' := \frac{\lambda_2 \omega}{1024 \theta}\,.$$ 
    By definition of $\lambda_2$, for a constant $M_3$ we have $\lambda'_2 \geq 4 \ln \left(\tfrac{8N}{\delta}\right)$. Due to the multiplicative Chernoff Bound (Lemma~\ref{lem:mult-chern-bound}), we derive
    \begin{equation} \label{eq:active-PU-1}
        \Pr[\,|R_2| \leq \lambda'_2\,] \leq \exp\!\left(-\tfrac{\lambda'_2}{4}\right) \leq \tfrac{\delta}{8N}.
    \end{equation}
    Assuming $|R_2| > \lambda'_2$, by the definition of $\lambda_2$ there exists a constant $M_3$ such that
    $$
    (5M_1 + 3M_2)  \frac{d \ln(2 |R_2| / d)  + \ln(8N/\delta)}{|R_2|} \leq \frac{1}{2 \theta}\,.
    $$
Hence, by \eqref{eq:active-PU-claim-3-1}
\[
    V_{i + 1} \subseteq \mathcal{B}_\D\!\left(\ell, \frac{\Delta_\D(V_i)}{2\theta}\right).
\]
This implies
\[
    \Delta_\D(V_{i+1}) \leq \Delta_\D\!\left(\mathcal{B}_\D\!\left(\ell, \tfrac{\Delta_\D(V_i)}{2\theta}\right)\right)
    \; \leq \; \tfrac{1}{2}\Delta_\D(V_i).
\]
By taking a union bound over all iterations $i = 0, \ldots, N-1$, Claim~3 holds for every $i \in \{0, \ldots, N-1\}$. 

We now complete the proof. Observe that $u_i - b_i$ is a decreasing function of $i$: each time the algorithm enters the \texttt{if} statement in Line~\ref{line:active-PU-if-u-small}, the value of $u_i - b_i$ is halved. Since, prior to termination, $u_i - b_i$ always remains greater than $\varepsilon/6$, the algorithm can enter this \texttt{if} branch at most $\log_2(6/\varepsilon)$ times.

Similarly, observe that prior to termination, $\hat{\Delta}_{\hat{S}_1}(V_i)$ remains greater than $\varepsilon/3$. Since
\[
|S_1| \geq \frac{24 \ln(8 N / \delta)}{\varepsilon},
\]
by the multiplicative Chernoff bound (Lemma~\ref{lem:mult-chern-bound}), with probability at least $1 - \frac{\delta}{8N}$, it holds that
\begin{equation*} 
    \Delta_{\D}(V_i) \geq \frac{\varepsilon}{6}\,.
\end{equation*}
Applying a union bound for all iterations $i = 0, 1, ..., N$ such that the algorithm enters the \texttt{while} loop, with probability $1- \delta / 8$ we get that $\Delta_{\D}(V_i)$ remains greater than $\frac{\varepsilon}{6}$ for all $i \in \{0, \ldots, N-1\}$. Observe that each time the algorithm enters the \texttt{else} statement in Line~\ref{line:active-PU-else}, $\Delta_\D(V_i)$ is divided by two. Thus, with probability at least $1 - \delta / 8$, the number of times the algorithm enters the \texttt{else} in Line~\ref{line:active-PU-else} is at most $\log_2(6 / \varepsilon)$.

Hence, the algorithm halts in $N = 2\log_2(6 / \varepsilon)$ iterations. Suppose the algorithm halts at iteration $n$. Then, the empirical error of its output is bounded from above by
\begin{equation*} 
    \hat \err_{S_1}(h_{n}, \ell) \leq \min\!\big(2(u_{n} - b_{n}), \hat \Delta_{S_1}(V_{n})\big) \leq \varepsilon / 3\,.
\end{equation*}
Finally, note that $S_1$ satisfies
$$
M_2 \frac{d \ln(2|S_1| / d) + \ln(8 / \delta)}{|S_1|} \leq \frac{\varepsilon}{3}\,.
$$
Thus, by Corollary~\ref{cor:B-hat-B}(i), with probability at least $1 - \delta/8$, we have
\[
    \err_\D(h_{n}, \ell) \leq 2 \hat \err_{S_1}(h_{n}, \ell) + \varepsilon / 3 \leq \varepsilon.
\]
Taking a union bound over events assumed true in this argument establishes that the learner outputs a low-error hypothesis with high probability, as claimed in the theorem. It remains to bound the number of queries made by the learner.

The number of queries makes for estimating $\hat \omega$ is at most
\[
O\!\left(\frac{\theta\bigl(d \ln(\theta / \pi_\D)+\ln(1/\delta)\bigr)}{\pi_{\D}^{2}\,\omega}\right).
\]
Afterwards, in each iteration the learner makes at most $\lambda_1 + \lambda_2$ queries, thus making $N(\lambda_1 + \lambda_2)$ queries in total. This completes the proof.

\end{proof}

\section{Conclusion and Open Problems}
\label{sec:conclusion}

This paper provides the first formal label complexity analysis for active learning from positive and unlabeled data. Our main result establishes that the label complexity of active PU learning can be bounded by
\[
\begin{aligned}
O\Bigg(
\frac{\ln (1 / \varepsilon)\,\theta^2 \,
\big(d \ln(\theta) + \ln \ln(1 / \varepsilon) + \ln(1 / \delta)\big)}{\omega}
\;+\; \\
\frac{\theta\big(d \ln(\theta / \pi_\D) + \ln(1 / \delta)\big)}{\pi_{\D}^2\,\omega}
\Bigg).
\end{aligned}
\]
Ignoring the second term, which is independent of $\varepsilon$, this bound differs from the classical active learning label complexity bound of \cite{hanneke2009theoretical} for the CAL algorithm,
\[
\ln (1 / \varepsilon)\,\theta \,
\big(d \ln(\theta) + \ln \ln(1 / \varepsilon) + \ln(1 / \delta)\big),
\]
by only a factor of $\theta / \omega$. The dependence on $\omega$ in our bound is inherent to the PU setting, since the effective rate of obtaining labels is always proportional to $\omega$. Whether the $\theta^2$ dependence in our active PU learning bound is fundamental, or can be reduced to $\theta$, remains an open question.

The additional $\theta$ factor in our bound stems from the fact that, in our analysis, the rate at which positive instances are queried is at most $1/\theta$. A similar phenomenon is already observed in passive PU learning. In particular, \citet{mansouri2025learning} established that the number of positive examples required for learning is bounded from below by
\[
\Omega\!\left(\frac{d + \ln(1/\delta)}{\varepsilon}\right).
\]
However, collecting this many random positive examples requires querying as many as
\[
\Omega\!\left(\frac{d + \ln(1/\delta)}{\pi_\D \varepsilon}\right)
\]
random pool examples, incurring an unavoidable $1/\pi_\D$ overhead in terms of sample complexity compared to fully supervised passive learning. 

A second direction for future work is the study of active PU learning under discrete distributions. In such settings, a fundamental difficulty arises: if querying an instance with large probability mass yields no feedback, then the learner may no longer be able to extract additional information from that instance, rendering learning impossible.
One possible approach to address this issue is to allow the learner, at test time, to refrain from labeling instances that received the feedback $\star$ during training. However, this introduces a new challenge, as the i.i.d.\ property on the selected instances (see Remark~\ref{remark:cont-asmp}) would no longer hold.

A third problem left for future research is to analyze the label complexity of active PU learning in the general agnostic setting.

\nocite{langley00}

\bibliography{ref}
\bibliographystyle{icml2026}

\newpage
\appendix
\section{Proof of Theorem~\ref{thm:estrate}}
\EstRate*
\begin{proof}
    We first prove that the algorithm halts within $O\left( \frac{ \theta(d  + \ln(1 / \delta)) }{\pi^2_D \omega} \right)$ label requests with high probability. Then we show that when it halts, it satisfies the guarantees it required with high probability.

   For an $M > 1$ we setup later, let $i^*$ be the first iteration such that 
   $$2^{i^*} \geq  \frac{ M \theta(d \ln(14 \theta / \pi_\D) + \ln(4 / \delta)) }{ \pi_\D^2 \omega}.
   $$ Note that since $\D$ is continuous almost surely all there is no repetition among $R$.  Then, using Lemma~\ref{lem:mult-chern-bound}
   , as long as $M \geq 8$
   $$
   \Pr \left[|R| \leq \frac{2^{i^*} \pi_\D \omega}{2} \right] \leq e^{- \frac{2^{i^*} \pi_\D \omega}{8}} < \delta / 4\,.
   $$
   Thus, with probability $1 - \delta / 4$ we have 
   \begin{equation} \label{eq:estrate-r1}
       |R| >  \frac{ M \theta(d \ln(14 \theta / \pi_\D) + \ln(4 / \delta)) }{ 2 \pi_\D} \,.
   \end{equation}

   Next note that $V$ defined in Line~\ref{line:estrate-V} of the algorithm
   is the set of all hypotheses that are consistent with $R$. Moreover, $h^*$ defined in Line~\ref{line:estrate-h} is the hypothesis in $V$ which predicts the fewest positive labels over $S$, i.e., corresponds to $h^{PU}$ in Theorem~\ref{thm:real-passive-PU}, with the positive labeled sample $R$, and the unlabeled sample $S$.
   
   Consider $\gamma$ defined in Line~\ref{line:estrate-gamma} of the algorithm. Due to Remark~\ref{remark:cont-asmp}, the responses are i.i.d. samples from $\D(. \mid y = 1)$. Therefore, using Theorem~\ref{thm:real-passive-PU}, with probability $1 - \delta / 4$ we have
   $$
    \rho_\D(h^*, \ell) < \gamma.
   $$


   Since $|S| \geq |R|$ and due to definition of $\gamma$ we have
   $$
   M_2 \frac{d \ln(2|S| / d) + \ln(4 / \delta)}{|S|} < \gamma
   $$
   Thus, we can use Corollary~\ref{cor:B-hat-B} (ii). Which with probability $1 - \delta / 4$ yields
   \begin{equation}
       \mathcal B_{\D} \left(h^*, \gamma \right) \subseteq \hat{\mathcal{B}}_S \left(h^*, 3 \gamma \right) \subseteq \mathcal{B}_\D \left(h^*, 7 \gamma \right)\,.
   \end{equation}
    Thus, $\ell \in\mathcal{B}_\D(h^*, \gamma) \subseteq \hat{\mathcal{B}}_S(h^*, 3 \gamma)$.


    Due to \eqref{eq:estrate-r1} and definition of $\gamma$ for $M \geq 28 (M_1 + M_2)$ we have $\gamma \leq \frac{\pi_\D}{14 \theta}$. Therefore, using the definition of $\theta$ we derive
    $$
    \begin{aligned}
        \frac{2 \theta \Delta_\D\left(\hat{\mathcal{B}}_S(h^*, 3 \gamma)\right)}{  \pi_\D  } & \geq \frac{2 \theta \Delta_\D\left(\mathcal{B}_\D(h^*, 7 \gamma)\right)}{  \pi_\D }  \\
        & \geq 
        \frac{2 \theta \Delta_\D\left(\mathcal{B}_\D(h^*, \frac{\pi_\D}{2 \theta})\right)}{  \pi_\D }  
        \geq \theta\,.
    \end{aligned}
    $$
    Thus $\Delta_\D\left(\hat{\mathcal{B}}_S(h^*, 3\gamma)\right) \geq \pi_\D / 2$. Denoting 
    $$E = \{x \in \X \mid \forall h' \in \hat{\mathcal{B}}_S(h^*, 3\gamma): h'(x) = 1\},$$ 
    this indicates that
    $$
    \begin{aligned}
        \Pr(E) & = \pi_\D - \Pr \left (\DIS^P(\hat{\mathcal{B}}_S(h^*, 3\gamma))\right) \\
        & \geq \pi_\D - \Pr\left(\DIS(\hat{\mathcal{B}}_S(h^*, 3\gamma))\right) \geq \frac{\pi_\D}{2}.
    \end{aligned}
    $$

    Due to Remark~\ref{remark:cont-asmp} all $r$ members of $P$ with feedback $1$ are i.i.d.\ samples from $E$. For $M \geq 48$ we have $|S| = 2^{i^*} > \frac{48 \ln(4 / \delta)}{\omega \pi_\D}$. Thus, using Lemma~\ref{lem:mult-chern-bound}, since, we have
     $$
     \Pr[ r \leq 12 \ln(4 / \delta)] \leq \exp \left (- \frac{24 \ln(4 / \delta)}{8} \right) < \delta / 4\,.
     $$
    Taking a union bound over all events that we assumed to hold, we conclude that with probability at least $1 - 3\delta/4$, the algorithm halts by iteration $i^*$. Note that since $|P| \leq |S|$, at each iteration $i$ the learner makes at most $2^{i+1}$ queries. Therefore, the total number of queries is at most
    \[
    \sum_{i = 0}^{i^*} 2^{i+1} < 4 \cdot 2^{i^*}
    = O\!\left(\frac{\theta\bigl(d + \ln(1/\delta)\bigr)}{\pi_\D^2\,\omega}\right).
    \]

     Finally we need to prove that $\frac{r}{|P|} \in [\omega / 2, 2 \omega]$. Notice the algorithm gets out of the while  stops when $r \geq 8 \ln(8 / \delta)$. Moreover, thus, $|P|$ would be $NB(r, \omega) + r$ where $NB$ is the negative binomial random variable. Thus $|P|$ has average $\frac{r}{\omega}$. Thus, again using Lemma~\ref{lem:mult-chern-bound} we derive
     $$
     \Pr \left [ |P| \not \in \left [\frac{r}{2\omega}, \frac{2r}{\omega} \right]\right ] = 2 \exp \left( -\frac{r}{8 \omega} \right) \leq \delta / 4\,.
     $$
     This completes the proof.
\end{proof}



\end{document}

%% file: macros.tex
\usepackage[utf8]{inputenc} 
\usepackage[T1]{fontenc}    
\usepackage{hyperref}       
\usepackage{url}            
\usepackage{booktabs}       
\usepackage{amsfonts}       
\usepackage{nicefrac}       
\usepackage{microtype}      
\usepackage{xcolor}         

\usepackage{graphicx} 
\usepackage{amsthm}
\usepackage{amsmath,amsfonts,amssymb,mathtools, bbm}
\usepackage{xcolor}
\usepackage{multirow}

\newcommand{\cH}{\mathcal{H}}
\newcommand{\X}{\mathcal{X}}

\newcommand{\D}{\mathcal{D}}

\newcommand{\argmin}{\operatorname{argmin}}
\newcommand{\VCD}{\operatorname{VCD}}

\newcommand{\DIS}{\operatorname{DIS}}

\newcommand{\err}{\operatorname{err}}